%% file: main.tex
\documentclass{article} 
\usepackage{iclr2015,times}
\usepackage{hyperref}
\usepackage{url}
\usepackage{amsmath,amsbsy,amsfonts,amssymb,amsthm,dsfont,units}
\usepackage{graphicx}
\usepackage{mathtools,appendix}
\usepackage{color,cases}
\usepackage{tikz}
\usetikzlibrary{calc,shapes}

\usepackage{algorithm}
\usepackage{algorithmic,natbib}
\usepackage{color,cases}
\usepackage{tikz}
\usepackage{enumitem}

 \newcommand{\IGNORE}[1]{}

\renewcommand{\cite}{\citep}

\newcommand\D{\operatorname{Diag}}

\input{macros}

\iclrfinalcopy
\usepackage{lipsum}

\title{Provable Methods for Training Neural Networks with Sparse Connectivity}

\author{
Hanie Sedghi \\
University of Southern California\\
Los Angeles, CA 90089 \\
\texttt{hsedghi@usc.edu}\\
\And
Anima Anandkumar \\
University of California \\
Irvine, CA 92697 \\
\texttt{a.anandkumar@uci.edu}\\
}

\begin{document}

\maketitle

\input{intro-workshop}

\input{summary-workshop}

\input{related-deep}

\input{results_single}

\input{conclusion-workshop}



\subsection*{Acknowledgment}
A. Anandkumar
is supported in part by Microsoft Faculty Fellowship, NSF Career award CCF-$1254106$, NSF Award CCF-$1219234$, ARO YIP Award W$911$NF-$13$-$1$-$0084$ and ONR Award N$00014-14-1-0665$. H. Sedghi is supported by ONR Award N$00014-14-1-0665$.

\end{document}

%% file: macros.tex
 \DeclareMathOperator*{\argmin}{arg\,min}

\DeclareMathOperator{\rank}{Rank}

 \def\0{{\bf 0}}

\def\qed{\hfill\hbox{${\vcenter{\vbox{
    \hrule height 0.4pt\hbox{\vrule width 0.4pt height 6pt
    \kern5pt\vrule width 0.4pt}\hrule height 0.4pt}}}$}}

\def\tcr{\textcolor{red}}
\def\tcb{\textcolor{blue}}

\definecolor{myred}{rgb}{0.3,0.0,0.7}
\definecolor{dkg}{rgb}{0.1,0.7,0.2}
\definecolor{dkb}{rgb}{0.0,0.2,0.8}

\def\tcdkg{\textcolor{dkg}}
\definecolor{lpurple}{cmyk}{.05,0.18,0,0}

\def\Nc{{\cal N}}

\def\Ebb{{\mathbb E}}

\newcommand{\bprf}{\begin{myproof}}
\newcommand{\eprf}{\end{myproof}}
\newcommand{\bp}{\begin{psfrags}}
\newcommand{\ep}{\end{psfrags}}
\newcommand{\bl}{\begin{lemma}}
\newcommand{\el}{\end{lemma}}
\newcommand{\bt}{\begin{theorem}}
\newcommand{\et}{\end{theorem}}
\newcommand{\bc}{\begin{center}}
\newcommand{\ec}{\end{center}}
\newcommand{\bi}{\begin{itemize}}
\newcommand{\ei}{\end{itemize}}
\newcommand{\ben}{\begin{enumerate}}
\newcommand{\een}{\end{enumerate}}
\newcommand{\bd}{\begin{definition}}
\newcommand{\ed}{\end{definition}}
\def\beq{\begin{equation}}
\def\eeq{\end{equation}\noindent}
\def\beqn{\begin{eqnarray}}
\def\eeqn{\end{eqnarray} \noindent}
\def\beqnn{  \begin{eqnarray*}}
\def\eeqnn{\end{eqnarray*}  \noindent}
\def\bcase{  \begin{numcases}}
\def\ecase{\end{numcases}   \noindent}
\def\bsbcase{  \begin{subnumcases}}
\def\esbcase{\end{subnumcases}   \noindent}

\newtheorem{theorem}{Theorem}

\newtheorem{lemma}{Lemma}

\newtheorem{proposition}{Proposition}

\newtheorem{definition}{Definition}

\newtheorem{remark}{Remark}

\newenvironment{myproof}{\noindent{\bf Proof:} \hspace*{1em}}{
    \hspace*{\fill} $\Box$ }
\newenvironment{proof_of}[1]{\noindent {\bf Proof of #1: }}{\hspace*{\fill} $\Box$ }

\newcommand{\matplottc}[1]{               
        \unitlength .45truein
        \begin{center}
        \includegraphics{#1.ps}
        \end{picture}
        \end{center}
}

\def\psfancypar#1#2{\begingroup\def\par{\endgraf\endgroup\lineskiplimit=0pt}
               \setbox2=\hbox{\large\sc #2}
               \newdimen\tmpht \tmpht \ht2 \advance\tmpht by \baselineskip
               \font\hhuge=Times-Bold at \tmpht
               \setbox1=\hbox{{\hhuge #1}}
               \count7=\tmpht \count8=\ht1
               \divide\count8 by 1000 \divide\count7 by \count8
               \tmpht=.001\tmpht\multiply\tmpht by \count7
               \font\hhuge=Times-Bold at \tmpht
               \setbox1=\hbox{{\hhuge #1}}
               \noindent
                \hangindent1.05\wd1
               \hangafter=-2 {\hskip-\hangindent
               \lower1\ht1\hbox{\raise1.0\ht2\copy1}%
                \kern-0\wd1}\copy2\lineskiplimit=-1000pt}

\def\Kout{\setbox1=\hbox{\Huge\bf K}\hbox to
1.05\wd1{\hspace{.05\wd1}
}}
\def\Sout{\setbox1=\hbox{\Huge\bf S}\hbox to 1.05\wd1{\hspace{.05\wd1}
}}



%% file: intro-workshop.tex
\begin{abstract}
We provide  novel guaranteed approaches for training  feedforward neural networks with sparse connectivity. We leverage on the techniques  developed previously for learning linear   networks  and show that they can also be effectively adopted to  learn non-linear networks. We  operate on the moments involving label and the score function of the input, and show that their factorization provably yields the weight matrix of the first layer of a deep network under mild conditions. In practice, the output of our method can be employed as effective initializers for gradient descent.
\end{abstract}

\paragraph{Keywords: }Deep feedforward networks, sparse connectivity, $\ell_1$-optimization, Stein's lemma.

\section{Introduction}
The paradigm of deep learning has revolutionized our ability to perform challenging classification tasks in a variety of domains such as computer vision and speech. However, so far, a complete theoretical understanding of deep learning is lacking. Training deep-nets is a highly non-convex problem involving millions of variables, and   an exponential number of fixed points. Viewed naively, proving any guarantees   appears to be  intractable.  In this paper, on the contrary, we show that  guaranteed learning of a subset of parameters is possible under  mild conditions.

We propose a novel learning algorithm based on the  method-of-moments. The notion of using moments for learning distributions dates back to Pearson~\cite{Pearson94}.  This paradigm has seen a recent revival in machine learning  and has been applied   for unsupervised learning of a variety of latent variable models (see~\cite{AnandkumarEtal:tensor12} for a survey). The basic idea is to develop efficient  algorithms for factorizing moment matrices and tensors. When the underlying factors are sparse,  $\ell_1$-based convex optimization techniques  have been proposed before, and been   employed for learning dictionaries~\cite{Spielman-12},   topic models, and linear latent Bayesian networks~\cite{AnandkumarEtal:DAG12}.

In this paper, we employ the $\ell_1$-based optimization method  to learn deep-nets with sparse connectivity. However, so far, this method has theoretical guarantees only for  linear models.   We develop novel techniques to prove the correctness even for non-linear models. A key technique we use is  the Stein's lemma from statistics~\cite{stein1986approximate}. Taken together, we show how to effectively leverage algorithms based on method-of-moments to  train deep non-linear networks.


%% file: summary-workshop.tex
\subsection{Summary of Results}

We present a theoretical framework for analyzing when neural networks can be learnt efficiently. We demonstrate how the method-of-moments can yield useful information about the weights in a neural network, and also in some cases, even   recover them exactly. In practice, the output of our method can be used for   dimensionality reduction for back propagation, resulting in reduced computation.

We show that in a feedforward neural network, the relevant moment matrix to consider is the cross-moment  matrix between the label and the score function of the input data (i.e. the derivative of the log of the density function).
The classical Stein's result~\cite{stein1986approximate} states that this matrix yields the expected derivative of the label (as a function of the input). The Stein's result   is essentially obtained through integration by parts~\cite{nourdin2013integration}.

By employing the Stein's lemma, we show that the  row span of the moment matrix between the label and the input score function  corresponds to the span of the weight vectors in the first layer, under natural non-degeneracy conditions.
Thus, the singular value decomposition of this moment e matrix can be used as low rank approximation  of the first layer weight matrix during back propagation, when the number of neurons is less than the input dimensionality. Note that since the first layer typically has the most number of parameters (if a convolutional structure is not assumed),
having a low rank approximation results in significant improvement in  performance and computational requirements.

We then show that we can exactly recover the weight matrix of the first layer from the moment matrix, when the weights are sparse. It has been argued that sparse connectivity is a natural constraint which can lead to improved performance in practice~\cite{thom2013sparse}. We show that the weights can be correctly recovered using an efficient $\ell_1$ optimization approach. Such approaches have been earlier employed for linear models such as dictionary learning~\cite{Spielman-12} and topic modeling~\cite{AnandkumarEtal:DAG12}. Here, we establish that the method is also successful in learning non-linear networks, by alluding to Stein's lemma.

Thus, we show that  the  cross-moment matrix between the label and  the score function of the input contains useful information for training neural networks. This  result has an intriguing connection with~\cite{alain2012regularized}, where it is shown a denoising auto-encoder approximately learns the score function of the input. Our analysis here provides a theoretical explanation of why pre-training can lead to improved performance during back propagation: the interaction between the score function (learnt during pre-training) and the label during back propagation results in correctly identifying the span of the  weight vectors, and thus, it leads to improved performance.

The use of  score functions for improved classification performance is popular under the framework of Fisher kernels~\cite{jaakkola1999exploiting}. However, in~\cite{jaakkola1999exploiting}, Fisher kernel is defined as the derivative with respect to some model parameter, while here we consider the derivation with respect to the input and refer to it as score function. Note that if the Fisher kernel is with respect to a location parameter, these two notions are equivalent. Here, we show that considering the moment between the label and the  score function of the input can lead to guaranteed learning and improved classification.

Note that there are various efficient methods for computing the score function (in addition to the auto-encoder). For instance,~\citet{sasaki2014clustering} point out that  the score function can be estimated efficiently through non-parametric methods without the need to estimate the density function. In fact, the solution is closed form, and the hyper-parameters (such as the kernel bandwidth and the regularization parameter) can be tuned easily through cross validation. There are a number of score matching algorithms, where the goal is to find a good fit in terms of the score function, e.g~\cite{hyvarinen2005estimation, swersky2011autoencoders}.  We can employ them to obtain accurate estimations of the score functions.

Since we employ a method-of-moments approach, we assume that the label is generated by a feedforward neural network, to which the input data is fed. In addition, we make mild non-degeneracy assumptions on the weights and the derivatives of the activation functions. Such assumptions  make the learning problem tractable, whereas   the general learning problem is NP-hard. We expect that the output of our moment-based approach can provide effective initializers for the back propagation procedure. 


%% file: related-deep.tex
\subsection{Related Work}



In this paper, we show that the method-of-moments can yield low rank approximations for weights in the first layer.
Empirically,  low rank approximations of the weight matrices have been employed successfully to  improve the performance and for reducing computations~\cite{davis2013low}. Moreover, the notion of using moment matrices for   dimension reduction is popular in statistics, and the dimension reducing subspace  is termed as a central subspace~\cite{cook1998principal}.

We present a $\ell_1$ based convex optimization technique to learn the weights in the first layer, assuming they are sparse. Note that this is different from other convex approaches for learning feedforward neural network. For instance,~\citet{bengio2005convex} show via a boosting approach that learning neural networks is a   convex optimization problem  as long as the number of hidden units can be selected by the algorithm. However, typically, the neural network architecture is fixed, and in that case, the optimization is non-convex.

Our work is the first to show guaranteed learning of a feedforward neural network incorporating both the label and the input.~\citet{arora2013provable} considered the auto-encoder setting, where learning is unsupervised,  and showed how the weights can be learnt correctly under a set of conditions. They assume that the hidden layer can be decoded correctly using a ``Hebbian'' style rule, and they all have only binary states. We present a different approach for learning by using the moments between the label and the score function of the input.


%% file: results_single.tex
\section{Moments of a Neural Network}

\subsection{Feedforward network with one hidden layer}
We first consider a feedforward network with one hidden layer. Subsequently, we discuss how much this can be extended.
 Let $y$ be the label vector generated from the neural network and $x$ be the feature vector. We assume $x$ has a well-behaved continuous probability distribution $p(x)$ such that the score function $\nabla_x \log p(x)$ exists. The network is depicted in Figure~\ref{fig:FF}. Let

\begin{align}
\label{eq:nn1}
\Ebb[y|h]=\sigma_2(A_2 h ),~~~ \Ebb[h|x]=\sigma_1(A_1x).
\end{align}


This setup is applicable to both multiclass and multilabel settings. For multiclass classification $\sigma_2$ is the softmax function and for multilabel classification $\sigma_2$ is a elementwise  sigmoid function. Recall that multilabel classification refers to the case where each instance can have more than one (binary) label~\citep{bishop2006pattern,tsoumakas2007multi}.

  \begin{figure}
\begin{center}
\begin{tikzpicture}
  [
    scale=1.1,
    observed/.style={circle,minimum size={width("$x_{d_x}$")+2pt},inner
sep=0mm,draw=violet,fill=lpurple,line width=.5mm},
    hidden/.style={circle,minimum size=0.6cm,inner sep=1mm,draw=dkg,line width=.5mm},
        func/.style={circle,minimum size=0.6cm,inner sep=1mm,draw=blue,dashed, line width=.5mm},
        vdots/.style={min, node distance=.5mm},
  ]
  \node [func,name=f1] at ($(-2,0)$) {$\tcb{\sigma_1}$};
  \node [func,name=f2] at ($(-1,0)$) {$\tcb{\sigma_1}$};
  \node [func,name=fk] at ($(2,0)$) {$\tcb{\sigma_1}$};
 \node [func,name=fn] at ($(1,0)$) {$\tcb{\sigma_1}$};
    \node [func,name=g1] at ($(0,-4)$) {$\tcb{\sigma_{d}}$};

    \node [hidden,name=h1] at ($(-2,-1)$){};
  \node [hidden,name=h2] at ($(-1,-1)$){};
  \node [hidden,name=hn] at ($(1,-1)$){};
  \node [hidden,name=hk] at ($(2,-1)$){};
  \node [] at ($(-2.7,-1)$) {$\tcdkg{h_1}$};
      \node [hidden,name=h21] at ($(-2,-2.5)$) {}; 
  \node [hidden,name=h22] at ($(-1,-2.5)$){};
  \node [hidden,name=h2k] at ($(2,-2.5)$){};
  \node [hidden,name=h2n] at ($(1,-2.5)$){};
  \node [] at ($(-2.7,-2.5)$) {$\tcdkg{h_{d-1}}$};
  \node[observed,name=y1] at ($(-1.5,-5)$){$y_1$};
  \node [observed,name=yk] at ($(-.5,-5)$) {$y_2$};
  \node [observed,name=y2] at ($(1.5,-5)$) {$y_{n_y}$};
 \node [observed,name=x11] at ($(-2.5,1.5)$) {$x_1$}; 
   \node [observed,name=x1] at ($(-1.5,1.5)$) {$x_2$}; 
  \node [observed,name=x2] at ($(-0.5,1.5)$) {$x_3$}; 
  \node [observed,name=xd2] at ($(2.5,1.5)$) {$x_{n_x}$}; 
  \node [] at ($(-3.2,1.5)$) {$x$};
  \node [] at ($(-2.2,-5)$) {$y$};
  \node [] at ($(-2.25,-3.5)$) {\tcr{$A_d$}};
   \node [] at ($(-3.2,0.75)$) {\tcr{$A_1$}};
  \node at ($(0,-1)$) {$\dotsb$};
    \node at ($(0.5,-5)$) {$\dotsb$};
  \node at ($(0,0)$) {$\dotsb$};
   \node at ($(1,1.5)$) {$\dotsb$};
    \node at ($(0,-2.5)$) {$\dotsb$};
    \path (h1) -- (h21) node [black, midway, sloped] {$\dots$};
     \path (h2) -- (h22) node [black, midway, sloped] {$\dots$};
      \path (hk) -- (h2k) node [black, midway, sloped] {$\dots$};
       \path (hn) -- (h2n) node [black, midway, sloped] {$\dots$};
  \draw [red, line width=.5mm, ->] (h21) to (g1);
  \draw [red, line width=.5mm, ->] (h22) to (g1);
  \draw [red, line width=.5mm, ->] (h2k) to (g1);
   \draw [red, line width=.5mm, ->] (h2n) to (g1);
    \draw [red, line width=.5mm, <-] (f1) to (x1);
  \draw [red, line width=.5mm, <-] (f1) to (x11);
   \draw [red, line width=.5mm, <-] (f2) to (x11);
  \draw [red, line width=.5mm,<-] (f2) to (x2);
  \draw [red, line width=.5mm,<-] (fk) to (xd2);
  \draw [red, line width=.5mm, <-] (fk) to (x1);
  \draw [red, line width=.5mm, <-] (fk) to (x2);
  \draw [red, line width=.5mm,<-] (fn) to (xd2);
    \draw [blue, line width=.5mm, <-] (h1) to (f1);
      \draw [blue, line width=.5mm, , <-] (h2) to (f2);
        \draw [blue, line width=.5mm, <-] (hk) to (fk);
        \draw [blue, line width=.5mm, <-] (hn) to (fn);
              \draw [blue, line width=.5mm, <-] (y1) to (g1);
      \draw [blue, line width=.5mm, <-] (y2) to (g1);
        \draw [blue, line width=.5mm, <-] (yk) to (g1);

\end{tikzpicture}
\end{center}
\caption{\small Graphical representation of Feedforward model $\Ebb[h|x]=\sigma_1(A_1x)$, $\Ebb[y|h]=\sigma_2(A_2 h )$.} \label{fig:FF}
\end{figure}
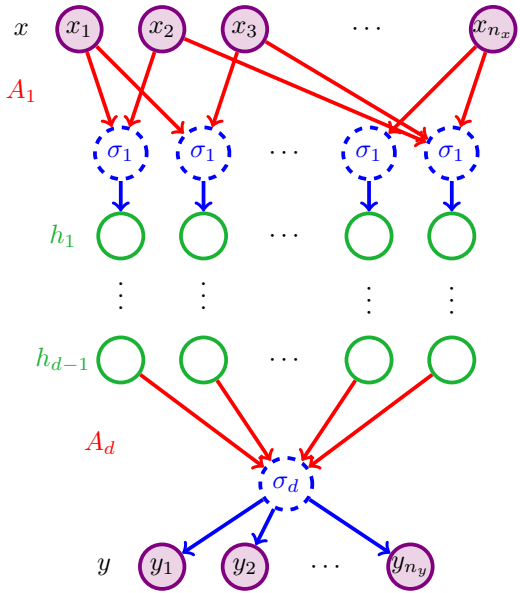

\subsection{Method-of-moments: label-score function  correlation matrix}
We hope to get information about the weight matrix using moments of the label and the input. The question is when this is possible and with what guarantees.
To study the moments let us start from a simple problem. For a linear network and whitened Gaussian input $x\sim \mathcal{N}(0,I)$, we have $y_{
\text{linear}}=Ax$. In order to learn $A$, we can form the label-score function   correlation matrix as
\begin{align*}
\Ebb[y_{\text{linear}}~x^\top]=A \Ebb[xx^\top]=A.
\end{align*}
Therefore, if $A$ is low dimensional, we can project $x$  into that span and perform classification in this lower dimension.

Stein's lemma for a Gaussian random vector $x$~\citep{stein1986approximate} states that for a function $g(\cdot)$ satisfying some mild regularity conditions we have
\begin{align*}
\Ebb[g(x) x^\top]=\Ebb_x[\nabla_x g(x)].
\end{align*}

A more difficult problem is generalized linear model (GLM) of a (whitened) Gaussian $x \in \mathbb{R}^{n_x}$. In this case, $y=\sigma(Ax)$ for any nonlinear activation function $\sigma(\cdot)$ that satisfies some mild regularity conditions. Using Stein's lemma we have
\begin{align*}
\Ebb[\sigma(Ax) x^\top]=\Ebb_{x'}[\nabla_{x'} \sigma(x'))]A,
\end{align*}where $x'\sim \Nc(0, AA^\top)$. 
Therefore, assuming $\Ebb_x[\nabla_x \sigma(Ax))]$ has full column rank, we obtain the row span of $A$. For Gaussian (and elliptical) random vector $x$, $P_{A^\top}x$ provides the sufficient statistic with no information loss. Thus, we can project the input into this span and obtain dimensionality reduction.

The Gaussian distribution assumption is a restrictive assumption. The more challenging problem is when random vector $x$ has a general probability distribution and the network has hidden layers. How can we deal with such an instance?  Below we provide the method to learn such problems.



\subsubsection{Results}
Let $x$ be a random vector with probability density function $p(x)$ and let $y$ be the output label corresponding to the network described in Equation~\eqref{eq:nn1}.
For a general probability distribution, we use score function of the random vector $x$ which provides us with sufficient statistics for $x$.
\paragraph{Definition: Score function} The score of $x$ with probability density function $p(x)$ is the random vector $\nabla_x \log p(x)$.


Let
\begin{align*}
M :=\Ebb[y \left(\nabla_x \log p(x)\right)^\top],
\end{align*}
which can be calculated in a supervised setting. Note that $\nabla_x \log p(x)$ represents the score function for random vector $x$.

\begin{theorem}
In a nonlinear neural network with feature vector $x$ and output label $y$, we have
\begin{align*}
M=-\Ebb_x[\sigma'_2(\tilde{x}_2) A_2 \D(\sigma'_{1}(\tilde{x}_1))] A_1,
\end{align*}
where $\tilde{x}_2=A_2 \sigma_1(A_1x)$ and $\tilde{x}_1=A_1x$.
\end{theorem}

\begin{proof}
Our method builds upon Stein's lemma~\cite{stein1986approximate}.
We use Proposition~\ref{steinslemma}.

\begin{align*}
M&=\Ebb_{x,y}[y \left(\nabla_x \log p(x)\right)^\top]=\Ebb_x\left[\Ebb_y\left[y \left(\nabla_x \log p(x)\right)^\top|x\right]\right] \\
&=\Ebb_x[\sigma_2(A_2(\sigma_{1}(A_1 x)\left(\nabla_x \log p(x)\right)^\top]  \\
&=-\Ebb_x[\sigma'_2(\tilde{x}_2) A_2 \D(\sigma'_{1}(\tilde{x}_1)) A_1]
\end{align*}
The second equality is a result of law of total expectation. The third equality follows from Stein's lemma as in Proposition~\ref{steinslemma} below. The last equality results from Chain rule. 
\end{proof}


\begin{proposition}[Stein's lemma~\citep{stein2004use}] \label{steinslemma}
Let $x \in \mathbb{R}^{n_x}$ be a random vector with joint density function $p(x)$. Suppose the score function $\nabla_x \log p(x)$ exists. Consider any continuously differentiable function $g(x):\mathbb{R}^{n_x} \rightarrow  \mathbb{R}^{n_y}$ such that all the entries of $g(x) p(x)^\top$ go to zero on the boundaries of support of $p(x)$. Then, we have
\begin{equation*}
\Ebb[g(x) \left(\nabla_x \log p(x)\right)^\top]=-\Ebb[\nabla_x g(x)],
\end{equation*}
Note that it is also assumed that the above expectations exist (in the sense that the corresponding integrals exist).
\end{proposition}
The proof follows integration by parts; the result for the scalar $x$ and scalar-output functions $g(x)$ is provided in~\citep{stein2004use}.

\begin{remark}[Connection with pre-training]
The above theorem provides us with a nice closed-form. If $B=\Ebb_x[\sigma'_2(\tilde{x}_2) A_2 \D(\sigma'_{1}(\tilde{x}_1))]$ has full column rank, we obtain the row space of $A_1$. In deep networks auto-encoder is shown to approximately learn the score function of the input~\cite{alain2012regularized}. It has been shown that pre-training results in better performance. Here, we are using the correlation matrix between labels and score function to obtain the span of weights.  Auto-encoder appears to be doing the same by estimating the score function. Therefore, our method  provides a theoretical explanation of why pre-training is helpful.
\end{remark}


\begin{remark} For whitened Gaussian (and elliptical) random vector,  projecting the input onto rowspace of $M$ is a sufficient statistic. Empirically, even for non-Gaussian distribution, this has lead to improvements~\citep{sun2013learning,li1992principal}. The moment method presented in this paper presents a low-rank approximation to train the neural networks. \end{remark}

So far, we showed that we can recover the span of $A_1$. How can we retrieve the matrix $A_1$? Without further assumptions this problem is not identifiable. A reasonable assumption is that $A_1$ is sparse.
In this case, we can pose this problem as learning $A_1$ given its row span. This problem arises in a number of settings such as learning a sparse dictionary or topic modeling. Next, using the idea presented in~\citep{Spielman-12}, we discuss how this can be done.


\section{Learning the Weight Matrix } \label{sec:learnA1}
In this Section, we explain  how we learn the weight matrix $A_1$ given the moment $M$. The complete framework is shown in Algorithm~\ref{algo:main}.
Assuming sparsity we use~\citet{Spielman-12} method. 

\begin{algorithm}[t]
\caption{Learning the weight matrix for the first layer of a Neural Network}
\label{algo:main}
\begin{algorithmic}[1]
\renewcommand{\algorithmicrequire}{\textbf{input}}
\renewcommand{\algorithmicensure}{\textbf{output}}
\REQUIRE Labeled samples $\{(x_i,y_i)\}, i \in [n]$.

\STATE Estimate  Score function $\nabla_x \log p(x) $ using auto-encoder or score matching.
\STATE Compute $\widehat{M}= \frac{1}{n} \sum_{i \in [n]} y_i \left(\nabla_x \log p(x) \right|_{x=x_i})^\top$ 
\STATE  $\hat{A}_1$=Sparse Dictionary Learning$(\widehat{M})$ 

\ENSURE  $\hat{A}_1$

\end{algorithmic}
\end{algorithm}

\paragraph{Identifiablity}  The first natural identifiability requirement on $A_1$ is  that it has full row rank. \citet{Spielman-12} show that for Bernoulli-Gaussian entries under relative scaling of parameters, we can impose that the sparsest vectors in the row-span of $M$ are the rows of $A_1$. Any vector in this space is generated by a linear combination $w^\top M$ of rows of $M$. The intuition is random sparsity, where a combination of different sparse rows cannot make a sparse row. Under this identifiability condition, we need to solve the optimization problem
\begin{align*}
\text{minimize}~~ \Vert w^\top M \Vert_0  ~~\text{subject to} ~~   w \neq 0.
\end{align*}
\paragraph{$\ell_1$ optimization} In order to come up with a tractable update,~\citet{Spielman-12} use the convex relaxation of $\ell_0$ norm and relax the nonzero constraint on $w$ by constraining it to lie in an affine hyperplane $\lbrace r^\top w=1 \rbrace$. Therefore, the algorithm includes solving the following linear programming problem
\begin{align*}
\text{minimize} ~~\Vert w^\top M \Vert_1  ~~\text{subject to}~~    r^\top w=1.
\end{align*}
It is proved that under some additional conditions, when $r$ is chosen as a column or sum of two columns of $M$, the linear program is likely to produce rows of $A_1$ with high probability~\citep{Spielman-12}. We explain these conditions in our context in Section~\ref{sec:Guarantees}.

By normalizing the rows of the output, we obtain a row-normalized version of $A_1$.
The algorithm is shown in Algorithm~\ref{algA1}. Note that $e_j$ refers to the $j$-th basis vector.

\begin{algorithm}[t]
\caption{Sparse Dictionary Learning~\citep{Spielman-12}.}
\label{algA1}
\begin{algorithmic}
\renewcommand{\algorithmicrequire}{\textbf{input}}
\renewcommand{\algorithmicensure}{\textbf{output}}
\REQUIRE $\widehat{M}$
\FOR {each $j=1,\dots,n_x$}
\STATE Solve $\min_w \|w^\top \widehat{M} \|_1$ subject to $(\widehat{M} e_j)^\top w=1$, and set $s_j=w^\top \widehat{M}$.
\ENDFOR
%
\STATE $\mathcal{S} =\lbrace s_1,\dotsc,s_{n_x}    \rbrace$
\FOR {each $i=1,\dots,k$}
\REPEAT
\STATE $l \leftarrow \argmin_{s_l \in \mathcal{S}} \|s_l\|_0$, breaking ties arbitrarily.
\STATE $v_i = s_l$.
\STATE $\mathcal{S} = \mathcal{S} \setminus \{s_l\}$.
\UNTIL{$\rank ([v_1,\dotsc,v_i])=i$}
\ENDFOR
\ENSURE Set $\hat{A}_1 = [\frac{v_1}{\Vert v_1 \Vert},\dotsc,\frac{v_k}{\Vert v_k \Vert}]^\top$. 
\end{algorithmic}
\end{algorithm}

We finally note that there exist more sophisticated analysis and algorithms for the problem of finding the sparsest vectors in a subspace. \citet{AnandkumarEtal:DAG12} provide the deterministic sparsity version of the result. \citet{barak2012hypercontractivity} require more computation and even quasi-polynomial time but they can solve the problem in denser settings.

\subsection{Guarantees for learning first layer weights} \label{sec:Guarantees}

We have the following assumptions to ensure that the weight matrix $A_1 \in \mathbb{R}^{k \times n_x} $ is learnt correctly.

 \paragraph{Assumptions}
 \begin{enumerate}
 \item[A.1]  \textbf{Elementwise first layer:} $\sigma_1$ is a elementwise function.

 \item[A.2] \textbf{Nondegeneracy:} $\Ebb_x[\sigma'_2(A_2 \sigma_1(A_1 x)) A_2 \D(\sigma'_{1}(A_1x))] $ has full column rank\footnote{Throughout this Section, we use the notation $\sigma'_{1}(A_1x)$ to denote $\sigma'_{1}(\tilde{x})\vert_{\tilde{x}=Ax}$.}.

\item[A.3] \textbf{Score function:} The score function $\nabla_x \log p(x)$ exists.

 \item[A.4]  \textbf{Sufficient input dimension:}  We have $n_x >c_1 k \log^4 k$ for some positive constant $c_1$.
\item[A.5] \textbf{Sparse connectivity:} The weight matrix $A_1 $ is Bernoulli$(\theta)$-Gaussian. For some positive constant $\alpha$, we have
 $\frac{2}{k} \leq \theta \leq \frac{\alpha}{\sqrt{k}}.$ 
\item[A.6] \textbf{Normalized weight matrix:} The weight matrix $A_1 $ is row-normalized.
 \end{enumerate}

\noindent Assumption A.1 is common in deep network literature since there are only elementwise activation in the intermediate layers.

Assumption A.2 is satisfied where $A_2$ is full-rank and $\sigma'_2(A_2 \sigma_1(A_1 x)), \D(\sigma'_{1}(Ax))$ are non-degenerate. This is the case when the number of classes is large, i.e. $n_y \geq k$ as in imagenets. In future, we plan to consider the setting with a small number of classes using other methods like tensor methods.
For non-degeneracy assumption of $\sigma'_2(\cdot)$, the reason is that we assume the functions are at least linear, i.e. their first order derivatives are nonzero. This is true for the activation function models in deep networks such as sigmoid function, piecewise linear rectifier and softmax function at the last layer.

Note that Assumption A.4 uses an improvement over Spielman's initial result~\citep{luh2015dictionary}. In a deep network $k$ is usually a few thousand while $n_x$ is in the millions. Hence, Assumption A.4 is satisfied. Note that~\citet{luh2015dictionary} have provided an algorithm for very sparse weight matrices, which only needs  $n_x >c_1 k \log k$.

Assumption A.5 requires the weight matrix to be sparse and the expected number of nonzero elements in each column of $A_1$ be at most $\mathcal{O}(\sqrt{k})$~\citep{luh2015dictionary}. In other words, each input is connected to at most $\mathcal{O}(\sqrt{k})$ neurons. This is a  meaningful assumption in the deep-nets literature as it has been argued that sparse connectivity is a natural constraint which can lead to improved performance in practice~\citep{thom2013sparse}.

If Assumption A.6 does not hold, we will have to learn the scaling and the bias through back propagation. Nevertheless, since the row-normalized $\hat{A}_1$ provides the directions, the number of parameters in back propagation is  reduced significantly.   Therefore, instead of learning a dense matrix we will only need to find the scaling in a sparse matrix. This results in significant shrinkage in the number of  parameters  the back propagation needs to learn. 

Finally we provide the results on learning the first layer weight matrix in a feedforward network with one hidden layer.

\begin{theorem}
Let Assumptions $A.1-A.5$ hold for the nonlinear neural network~\eqref{eq:nn1}, then Algorithm~\ref{algA1} uniquely recovers a row-normalized version of $A_1$ with exponentially small probability of failure.
\end{theorem}
For proof, see~\citep{Spielman-12}.

\begin{remark}[Efficient implementation] The $\ell_1$ optimization is an efficient algorithm to implement. The algorithm involves solving $k$ optimization problems. Traditionally, the $\ell_1$ minimization can be formulated as a linear programming problem. In particular, each of these $\ell_1$ minimization problems can be written as a LP with $2(k-1)$ inequality constraints and one equality constraint. Since the computational complexity of such a method is often too high for large scale problems, one can use approximate methods such as gradient projection~\citep{figueiredo2007gradient,kim2007interior}, iterative-shrinkage thresholding~\citep{daubechies2004iterative} and proximal gradient~\citep{nesterov1983method,nesterov2007gradient} that are noticeably faster~\citep{AnandkumarEtal:DAG12}.
\end{remark}

 \begin{remark}[Learning $\hat{A}_2$] After learning $A_1$, we can encode the first layer as $h=\sigma_1(A_1 x)$ and perform softmax regression to learn $A_2$.
 \end{remark}


\begin{remark}[Extension to deterministic sparsity]
The results in this work are proposed in the random setting where the i.i.d.\ Bernoulli-Gaussian entries for matrix $A_1$ are assumed. In general, the results can be presented in terms of deterministic conditions as in~\citep{AnandkumarEtal:DAG12}. \citet{AnandkumarEtal:DAG12} show that the model $M=BA_1$ is identifiable when $B$ has full column rank and the following expansion condition holds~\citep{AnandkumarEtal:DAG12}.
\begin{align*} 
|N_B(S)| \geq |S| + d_{\max}(B), \quad \forall S \subseteq \textnormal{Columns of $B$}, \ |S| \geq 2.
\end{align*}
Here, $N_B(S) := \{i \in [k]: B_{ij} \neq 0 \textnormal{\ for some \ } j \in S \}$ denotes the set of neighbors of columns of $B$ in set $S$. They also show that under additional conditions, the $\ell_1$ relaxation can recover the model parameters. See~\citep{AnandkumarEtal:DAG12} for the details.
\end{remark}

\subsection{Extension to deep networks}
So far, we have considered a network with one hidden layer. Now,
consider a deep $k$-node neural network with depth $d$. Let $y$ be the label vector and $x$ be the feature vector. We have
\begin{align}
\label{eq:nn2}
\Ebb[y|x]=\sigma_d(A_d \sigma_{d-1}(A_{d-1} \sigma_{d-2}(\cdots A_2 \sigma_1(A_1 x)))),
\end{align}
 where  $\sigma_1$ is elementwise function (linear or nonlinear). This set up is applicable to both multiclass and mutlilabel settings. For multiclass classification, $\sigma_d$ is the softmax function and for multilabel classification $\sigma_d$ is a elementwise  sigmoid function.
 In this network,  we can learn the first layer using the idea presented earlier in this Section to learn the first layer. From Stein's lemma, we have
 \begin{align*}
 M&=\Ebb[y \left(\nabla_x \log p(x)\right)^\top]=-\Ebb_x[\nabla_x y]\\&=\Ebb[\sigma'_d(\tilde{x}_d) A_d \sigma'_{d-1}(\tilde{x}_{d-1}) A_{d-1} \sigma'_{d-2}(\tilde{x}_{d-2}) A_{d-2} \cdots \sigma'_2(\tilde{x}_2) A_2 \D(\sigma'_1(\tilde{x}_1) )]A_1.
 \end{align*}

\textbf{Assumption B.2 Nondegeneracy:}

The matrix $B=\Ebb[\sigma'_d(\tilde{x}_d) A_d \sigma'_{d-1}(\tilde{x}_{d-1}) A_{d-1} \sigma'_{d-2}(\tilde{x}_{d-2}) A_{d-2} \cdots \sigma'_2(\tilde{x}_2) A_2 \D(\sigma'_1(\tilde{x}_1) )]$ has full column rank.

In Assumption B.2, $\tilde{x}_i=A_i h_i, i \in [d]$ where $h_i$ denotes the input and the $i$-th layer.

\begin{theorem}
Let Assumptions $A.1, B.2, A.3-A.6$ hold for the nonlinear deep neural network~\eqref{eq:nn2}. Then, Algorithm~\ref{algA1} uniquely recovers a row-normalized version of $A_1$ with exponentially small probability of failure.
\end{theorem}

The proof follows Stein's lemma, use of Chain rule and~\citep{Spielman-12}.

In a deep network, the first layer includes most of the parameters (if a structure such as convolutional  networks is not assumed) and other layers consist of a small number of parameters since there are small number of neurons. Therefore, the above result is a prominent progress in learning deep neural networks.

\begin{remark}[] This is the first result to learn  a subset of  deep networks for general nonlinear case in supervised manner.
The idea presented in~\citep{arora2013provable} is for the auto-encoder setting, whereas we consider supervised setting. Also,~\citet{arora2013provable}  assume that the hidden layer can be decoded correctly using a ``Hebbian'' style rule, and they all have only binary states. In addition, they can handle sparsity level up to $k^\gamma, 0 < \gamma \leq 0.2$ while we can go up to $\sqrt{k}$, i.e. $\gamma=0.5$.
\end{remark}

\begin{remark}[Challenges in learning the higher layers]
In order for $B$ to have  full column rank, intermediate layers should have square weight matrices. However, if we want to learn the middle layers, $A.4$ requires that the number of rows of the weight matrices be smaller than the number of columns in a specific manner and therefore $B$ cannot have full column rank.  In future, we hope to investigate new methods to help in overcoming this challenge.
\end{remark}




%% file: conclusion-workshop.tex
\section{Conclusion}
We introduced a new paradigm for learning neural networks using method-of-moments. In the literature, this method has been restricted to unsupervised setting. Here, we  bridged the gap and employed it for discriminative learning. 
This opens up a lot of interesting research directions for future investigation.
First, note  that we only considered the input to have continuous distribution for which the score function exists. The question is whether learning the parameters in a neural network is possible for the discrete data. Although Stein's lemma has a form for discrete variables (in terms of finite differences)~\citep{wei2010stein}, it is not clear how that can be leveraged to learn the network parameters. Next, it is worth analyzing  how we can go beyond $\ell_1$ relaxation and provide guarantees in such cases. Another interesting problem arises in case of small number of classes. Note that for non-degeneracy condition, we require the number of classes to be bigger than the number of neurons in the hidden layers. Therefore, our method does not work for the cases where $n_y < k$. In addition, in order to learn the weight matrices for intermediate layers,   we need the number of rows to be smaller than the number of columns to have sufficient input dimension. On the other hand, non-degeneracy assumption requires these weight matrices to be square matrices. Hence, learning the weights in the intermediate layers of deep networks is a challenging problem. It seems tensor methods, which have been highly successful in learning a wide range of hidden models such as topic modeling, mixture of Gaussian and community detection problem~\citep{AnandkumarEtal:tensor12}, may provide a way to overcome the last two challenges. 